\documentclass{article}

\usepackage{microtype}
\usepackage{graphicx}
\usepackage{subfigure}
\usepackage{booktabs} 

\usepackage{hyperref}

\usepackage[accepted]{icml2019}
 \usepackage{selectp}

\usepackage{tikz}

\usepackage[colorinlistoftodos, textwidth=22mm, shadow]{todonotes}
\definecolor{blued}{RGB}{70,197,221}

\definecolor{citrine}{rgb}{0.89, 0.82, 0.04}

\usepackage{nicefrac}
\usepackage{amsmath, amssymb,amsfonts,amsxtra,amsthm}
\usepackage{enumerate}

\definecolor{graphicbackground}{rgb}{0.96,0.96,0.8}
\definecolor{rouge1}{RGB}{226,0,38}  
\definecolor{orange1}{RGB}{243,154,38}  
\definecolor{jaune}{RGB}{254,205,27}  
\definecolor{blanc}{RGB}{255,255,255} 
\definecolor{rouge2}{RGB}{230,68,57}  
\definecolor{orange2}{RGB}{236,117,40}  
\definecolor{taupe}{RGB}{134,113,127} 
\definecolor{gris}{RGB}{91,94,111} 
\definecolor{bleu1}{RGB}{38,109,131} 
\definecolor{bleu2}{RGB}{28,50,114} 
\definecolor{vert1}{RGB}{133,146,66} 
\definecolor{vert3}{RGB}{20,200,66} 
\definecolor{vert2}{RGB}{157,193,7} 
\definecolor{darkyellow}{RGB}{233,165,0}  
\definecolor{lightgray}{rgb}{0.9,0.9,0.9}
\definecolor{darkgray}{rgb}{0.6,0.6,0.6}
\definecolor{babyblue}{rgb}{0.54, 0.81, 0.94}
\definecolor{citrine}{rgb}{0.89, 0.82, 0.04}
\definecolor{misogreen}{rgb}{0.25,0.6,0.0}

\DeclareMathOperator*{\argmax}{arg\,max}

\newcommand{\sset}[1]{\left\{#1\right\}}

\newcommand{\II}[1]{\mathbb{I}{\left\{#1\right\}}}

 \newtheorem{lemma}{Lemma}
 \newtheorem{theorem}{Theorem}
 \newtheorem{definition}{Definition}
 
\newtheorem{remark}{Remark}
 \newtheorem{proposition}{Proposition}
\newtheorem{fact}{Fact}
 \newtheorem{example}{Example}

\newcommand{\R}{\mathbb{R}}

\newcommand{\EE}[1]{\mathbb{E}\left[#1\right]}

\newcommand{\pa}[1]{\left(#1\right)}

\newcommand{\norm}[1]{\left\|#1\right\|}

\newcommand{\abs}[1]{\left|#1\right|}
\newcommand{\imp}{\Rightarrow}

\newcommand{\CommaBin}{\mathbin{\raisebox{0.5ex}{,}}}

\newcommand{\transpose}{^\mathsf{\scriptscriptstyle T}}

\newcommand{\cA}{\mathcal{A}}
\newcommand{\cB}{\mathcal{B}}
\newcommand{\cC}{\mathcal{C}}

\newcommand{\cI}{\mathcal{I}}

\newcommand{\cL}{\mathcal{L}}

\newcommand{\cO}{\mathcal{O}}

\newcommand{\cP}{\mathcal{P}}

\newcommand{\be}{{\bf e}}

\newcommand{\bX}{{\bf X}}

\renewcommand{\epsilon}{\varepsilon}

\renewcommand{\bar}{\overline}

\newcommand{\bdelta}{{\boldsymbol \delta}}

\newcommand{\bmu}{{\boldsymbol \mu}}

\newcommand{\nothere}[1]{}

\usepackage{xspace}

\allowdisplaybreaks
\newcommand{\rd}[1]{{#1}}

\newcommand{\arms}{n}
\newcommand{\actions}{\cA}

\newcommand{\kl}{\text{kl}}

\newcommand{\vmean}[1]{\bar{\rd{\bmu}}_{#1}}
\newcommand{\mean}[1]{\bar{\mu}_{#1}}

\newcommand{\counter}[2]{\rd{{N}}_{#1,#2}}
\newcommand{\blambda}{{\boldsymbol \lambda}}

\newcommand{\N}{\mathbb{N}}
\newcommand{\tableq}[1]{{$\!\begin{aligned} 
               \centering #1
                \end{aligned}$}}

\newcommand{\tabtab}[1]{\begin{tabular}{c}#1\end{tabular}}

\let\originalleft\left
\let\originalright\right
\renewcommand{\left}{\mathopen{}\mathclose\bgroup\originalleft}
\renewcommand{\right}{\aftergroup\egroup\originalright}

\icmltitlerunning{Exploiting Structure of Uncertainty for Efficient Matroid Semi-Bandits}
\begin{document}
\twocolumn[
\icmltitle{Exploiting Structure of Uncertainty for Efficient Matroid Semi-Bandits}




\begin{icmlauthorlist}
\icmlauthor{Pierre Perrault}{inria,ens}
\icmlauthor{Vianney Perchet}{ens,criteo}
\icmlauthor{Michal Valko}{inria,ndm}
\end{icmlauthorlist}

\icmlaffiliation{inria}{SequeL team, INRIA Lille - Nord Europe}
\icmlaffiliation{ens}{CMLA, ENS Paris-Saclay}
\icmlaffiliation{criteo}{Criteo AI Lab}
\icmlaffiliation{ndm}{now with DeepMind}

\icmlcorrespondingauthor{P.\,Perrault}{pierre.perrault@inria.fr}
\icmlkeywords{combinatorial semi-bandit, matroid, submodular, local search, greedy, efficiency,approximation}

\vskip 0.3in
]



\printAffiliationsAndNotice{}  

\begin{abstract}
We improve the efficiency of algorithms for stochastic \emph{combinatorial semi-bandits}.
In most interesting problems, state-of-the-art algorithms take advantage of structural properties of rewards, such as \emph{independence}. However, while 
being optimal in terms of asymptotic regret, these algorithms are inefficient. In our paper, 
we first reduce their implementation to a specific \emph{submodular maximization}. Then, in case of \emph{matroid} constraints,  
we design adapted approximation routines, thereby providing the first efficient algorithms that rely on reward structure to improve regret bound. In particular, we improve the state-of-the-art efficient gap-free regret bound by a factor $\sqrt{m}/\log m$, where $m$ is the maximum action size.
Finally, we show how our improvement translates to more general \emph{budgeted combinatorial semi-bandits}.
\end{abstract}
\vspace{-0.6cm}
\section{Introduction}
\label{sec:intro}
Stochastic bandits model sequential decision-making  in which an \emph{agent} selects an arm (a decision) at each round and observes a realization of the corresponding unknown reward distribution. The goal is to maximize the expected cumulative reward, or equivalently, to minimize the \emph{expected regret},  defined as the difference between the expected cumulative reward achieved by an oracle algorithm always selecting the optimal arm and that achieved by the agent. To accomplish this objective, the agent must trade-off between \emph{exploration} (gaining information about reward distributions) and \emph{exploitation} (using greedily the information collected so far) 
as it was already discovered by \citet{robbins1952some}. Bandits 
have been applied to many fields such as mechanism design \citep{mohri2014optimal}, search advertising \citep{tran2014efficient}, and personalized recommendation \citep{li2010contextual}. 
We improve the computational efficiency (i.e., the time and space complexity) of their \emph{combinatorial} generalization, in which the agent selects at each round a \emph{subset} of arms,  that we refer to as an  \emph{action} in the rest of the paper \citep{cesa-bianchi2012combinatorial,gai2012combinatorial,audibert2014regret,Chen2014}.

Different kinds of feedback provided by the environment are possible. 
First, with \emph{bandit feedback} (also called full bandit or opaque feedback), the agent only observes the \emph{total} reward associated to the selected action. Second,  with \emph{semi-bandit feedback}, the agent observes the \emph{partial} reward of each base arm in the selected action. Finally, with \emph{full information feedback}, the agent observes the partial reward of all arms. 
We give results for  semi-bandit feedback only.

There are two main questions that come up with combinatorial (semi)-bandits: 1$^\circ$ \emph{How can the stochastic structure of the reward vector  be exploited to reduce the regret?} and 2$^\circ$ \emph{Can algorithms be efficient?} \citet{combes2015combinatorial} answer the first question assuming that reward distributions are \emph{mutually independent}. Later,
\citet{Degenne2016} generalize the algorithm of  \citet{combes2015combinatorial} 
to a larger class of \emph{sub-Gaussian} rewards
by
exploiting the covariance structure of
the arms. They also show the optimality of proposed algorithms, in particular, that an upper bound on their regret matches the asymptotic gap dependent lower bound of this class.
However, algorithms of \citet{combes2015combinatorial} and \citet{Degenne2016} are computationally inefficient.
The second question is studied by \citet{kveton2015tight}, who give an efficient algorithm based on the \textsc{UCB} algorithm of \citet{auer2002finite}.
 While being efficient, the algorithm of \citet{kveton2015tight} assumes the worst case class of \emph{arbitrary correlated}\footnote{Any dependence can exist between rewards.} rewards, i.e., it does not  exploit any  properties of rewards and therefore does not match the lower bound of \citet{Degenne2016}. 

 On the other hand, efficient algorithms for matroid \citep{whitney35abstract} semi-bandits exist \citep{kveton2014matroid,Talebi2016} and their regret bounds match the asymptotic gap dependent lower bound, which is the same for both sub-Gaussian and arbitrary correlated rewards. Among these algorithms, the state-of-the-art gap-free regret bound is of order $\cO\pa{\sqrt{\arms m T\log T}}$, where $T$ is the number of rounds, $\arms$ is the number of base arms, and $m$ is the maximum action size  \citep{kveton2014matroid}.
 
 \paragraph{Our contributions}
 In this paper, we show how algorithms of \citet{combes2015combinatorial} and \citet{Degenne2016} can be \emph{efficiently} approximated for matroid. This improves the bound of \citet{kveton2014matroid} by a factor $\sqrt{m}/\log(m)$ for the class of \emph{sub-Gaussian} rewards. We first locate the source of inefficiency of these algorithms: At each round, they have to solve a \emph{submodular maximization} problem. We then provide efficient, adapted \textsc{LocalSearch} and \textsc{Greedy}-based algorithms that exploit the submodularity to give approximation guarantees on the regret upper bound. These algorithms can be of independent interest. 
We also extend our approximation techniques to more challenging \emph{budgeted combinatorial semi-bandits} via binary search methods
and exhibit the same improvement for this setting as well.  
 
\paragraph{Related work} 
Efficiency in combinatorial bandits, or more generally, in linear bandits with a large set of arms is an open problem. Some methods, such as \emph{convex hull representation}, \citep{Koolen2010}, \emph{hashing} \citep{Jun2017}, or \emph{DAG-encoding} \citep{Sakaue2018}
 reduce the algorithmic complexity. For semi-bandit feedback, in the adversarial case, \citet{Neu2013} proposed an efficient implementation  via \emph{geometric resampling}. In the stochastic case, many efficient Bayesian algorithms exist (see e.g., \citealp{Russo2013,russo2016}), although they are only shown to be optimal for \emph{Bayesian regret}.\footnote{A setting where arm distributions depend on a random parameter drawn from a known prior, and expectation of the regret is also taken with respect to this parameter.} 

\section{Background}
We denote the set of arms by $[\arms]\triangleq\sset{1,2,\dots,\arms}$,  we typeset vectors in bold and indicate components with indices, i.e., $\mathbf{a}=(a_i)_{i\in [\arms]} \in \R^\arms$.
We let $\cP([\arms])\triangleq\sset{A,A\subset [\arms]}$ be the power set of $[\arms]$. Let
$\be_i\in \R^\arms$ denote the $i^{\rm th}$  canonical  unit  vector. The \emph{incidence vector} of any set $A\in \cP([\arms])$ is
$$\be_A\triangleq\sum_{i \in A}\be_i.$$
The above definition allows us to represent a subset of $[\arms]$ as an element of $\sset{0,1}^\arms$.
We denote the Minkowski sum of two sets $Z,Z'\subset\R^\arms$ as $Z+ Z'\triangleq\sset{z+z',z\in Z,z'\in Z'}$, and  $Z+ z'\triangleq Z+\sset{z'}$. 
Let $\actions\subset\cP([\arms])$ be a set of \emph{actions}. We define the maximum possible cardinality of an element of $\actions$ as $m\triangleq\max\sset{\abs{A},~A\in \actions}\!.$ 
\subsection{ Stochastic Combinatorial Semi-Bandits}
In combinatorial semi-bandits, an agent selects an action $A_t\in \actions$ at each round $t\in\N^\star$, and receives a reward $\be_{A_t}\transpose \bX_t$, where $\bX_t\in \R^\arms$ is an unknown random vector of rewards. The successive reward vectors $\pa{\bX_t}_{t\geq 1}$ are i.i.d., with an unknown mean $\bmu^\star\triangleq\EE{\bX}\in \R^\arms$, where $\bX=\bX_1$.
 After selecting an action $A_t$ in round~$t$, the agent observes the partial reward of each individual arm in $A_t$. 
The goal of the agent is to
 minimize the expected regret, defined with $A^\star\in\argmax_{A\in \actions}\be_{A}\transpose\bmu^\star$ as 
\[R_T\triangleq\EE{\sum_{t=1}^T \pa{\be_{A^\star}-\be_{A_t}}\transpose\bX_t}.\]
For any action $A\in\actions$, we define its gap as the difference $\Delta\pa{A}\triangleq \pa{\be_{A^\star}-\be_{A}}\transpose\bmu^\star$. We then rewrite the expected cumulative regret as $R_T=\EE{\sum_{t=1}^T \Delta\pa{A_t}}$. Finally, we define $\Delta\triangleq \min_{A\in \actions,~\Delta\pa{A}>0}\Delta\pa{A}$.

Combinatorial semi-bandits have been introduced by \citet{cesa-bianchi2012combinatorial}. More recently, different algorithms have been proposed \citep{Talebi2013,combes2015combinatorial,kveton2015tight,Degenne2016}, depending whether the random vector $\bX$ satisfies specific properties. Some of these properties commonly assumed  are a subset of the following ones:
\begin{enumerate}[$(i)$]
\item $X_1,\dots,X_\arms\in \R$ are \emph{mutually independent},
 \item  $X_1,\dots,X_\arms\in \R$ are \emph{arbitrary correlated},
\item $\bX\in [-1,1]^\arms$,
\item $\bX\in\R^\arms$ is \emph{multivariate sub-Gaussian}, \\ i.e., $\EE{e^{\blambda\transpose\pa{\bX-\bmu^\star}}}\leq e^{\norm{\blambda}_2^2/2},\ \forall\blambda\in \R^\arms$,
\item $\bX\in\R^\arms$ is \emph{component-wise sub-Gaussian}, \\ i.e., $\EE{e^{\lambda_i\pa{X_i-\mu^\star_i}}}\leq e^{\lambda_i^2/2},\ \forall i\in[\arms],~\forall\blambda\in \R^\arms$. 
\end{enumerate}
\subsection{Lower Bounds}
Combining some of the above properties, we  consider different classes of possible distributions for $\bX$. In Table~\ref{table:lower}, we show two existing gap-dependent lower bounds on $R_T$ that depend on the respective class. They are valid for at least one distribution of $\bX$ belonging to the corresponding class, one combinatorial structure $\actions\subset \cP([\arms])$, and for any consistent algorithm \citep{lai1985asymptotically}, for which the regret on any problem verifies $R_T=o(T^a)$ as $T\to\infty$ for all $a>0$. Table~\ref{table:lower} suggests that a tighter regret rate can be reached with some prior knowledge on the random vector~$\bX$.

\begin{table}
\caption{Gap-dependent lower bounds proved on different classes of possible distributions for $\bX$.}
\vspace{0.1in}
\centering
\begin{tabular}{|c|c|}
\hline
\tabtab{\emph{Class of possible}\\\emph{ reward distributions}}
  & \tabtab{\emph{Gap-dependent} \\\emph{lower bound}} 
\\
 \hline
\tableq{&(i)+(iii)\\
\imp &(i)+(v)\\
\imp
 &(iv)}&\tabtab{\tableq{\Omega\pa{\frac{\arms\log T}{\Delta}}}\\\citealp{combes2015combinatorial}}
 \\ \hline
  \tableq{&(ii)+(iii)
 \\ \imp&
 (ii)+ (v)}&\tabtab{\tableq{\Omega\pa{\frac{\arms m\log T}{\Delta}}}\\\citealp{kveton2015tight}}
 \\ \hline
\end{tabular}
\label{table:lower}
\end{table}

\begin{table*}
\vskip -0.15in
\caption{Some combinatorial semi-bandit algorithms and their properties.}
\vspace{0.1in}
\centering
\begin{tabular}{|c|c|c|c|c|}
 \hline
\emph{Class} &\emph{Algorithm} & $p$ &$g_{i,t}$ (\emph{for} $\counter{i}{t-1}\geq 1$,\emph{ up to universal factor}) & \emph{Efficient?} \\
 \hline
 $(ii)+(v)$ &   \textsc{CUCB} \citep{kveton2015tight}& $\infty$&\tableq{\delta\mapsto\frac{{\delta}^2{\counter{i}{t-1}}}{\log(t)}}& Yes\\
 $(i)+(iii)$& \textsc{ESCB-KL} \citep{combes2015combinatorial} & 1 & \tableq{\delta\mapsto{\frac{\kl\pa{\frac{1+\mean{i,t-1}}{2}\CommaBin\frac{1+\mean{i,t-1}+\delta}{2}}\counter{i}{t-1}}{\log(t)+m
 }}}& No \\
 $(iv)$ &\tabtab{\textsc{ESCB} \citep{combes2015combinatorial}, \\\textsc{OLS-UCB} \citep{Degenne2016} }& 1&\tableq{\delta\mapsto{\frac{{\delta}^2\counter{i}{t-1}}{\log(t)+m
 }}} & No\\  \hline
\end{tabular}
\label{table:algos}
\end{table*}
  
\section{(In)efficiency of Existing Algorithms}
In this section, we discuss the efficiency of existing algorithms matching the lower bounds in Table~\ref{table:lower}. 
We consider that an algorithm is \emph{efficient} as  soon as the time and space complexity for each round $t$ is polynomial in $\arms$ and polylogarithmic\footnote{In streaming settings with near real-time requirements, it is imperative to have
algorithms that can run with a complexity that stay almost constant across rounds. } in~$t$.
Notice that the per-round complexity   depends substantially on~$\actions$.
We assume $\actions$ is such that linear optimization problems on $\actions$ --- of the form $\max_{A\in \actions} \be_A\transpose \bdelta$ for some $\bdelta\in \R^\arms$ --- can be solved efficiently. As a consequence, an agent knowing~$\bmu^\star$ can efficiently compute $A^\star$.
Assuming efficient linear maximization is crucial (cf.\,\citealp{Neu2013,combes2015combinatorial,kveton2015tight,Degenne2016}). Without this assumption, e.g., for $\actions$ being dominating sets in a graph, even the offline problem cannot be solved efficiently, and we would have to consider the notion of approximation regret instead, as was done by  \citet{CWY13}.
\subsection{A Generic Algorithm}
As mentioned above, the action $A_t$ is selected based on the feedback received up to round $t-1$. A common statistic computed from this feedback is the \emph{empirical average} of each arm $i\in [\arms]$, defined as
$$\mean{i,0}=0,\quad\forall t\geq 2,~\mean{i,t-1}=\frac{\sum_{u\in[t-1]}\II{i\in A_u}X_{i,u}}{\counter{i}{t-1}}\CommaBin$$
where $\forall t\geq 1,~\counter{i}{t-1}\triangleq\sum_{u\in[t-1]}\II{i\in A_u}$.
Many combinatorial semi-bandit algorithms, in particular, those listed in Table~\ref{table:algos}, can be seen as a special case of  Algorithm~\ref{algo:standard} for different confidence regions $\cC_t$ around $\vmean{t-1}$. 
\begin{algorithm}[H]
\caption{Generic confidence-region-based algorithm.} \label{algo:standard}
\begin{algorithmic}
\STATE At each round $t:$
\STATE \quad Find a confidence region $\cC_t\subset \R^\arms.$
\STATE \quad  Solve the bilinear program $$\pa{\bmu_t,{A_t}}\in\argmax_{\bmu\in \cC_t,A\in \actions} \be_A\transpose \bmu~.$$
\STATE \quad Play $A_t$.
\end{algorithmic}
\end{algorithm}

We further assume that $\cC_t$ is defined  through some parameters $p,r\in\sset{1,\infty}$, and some functions $g_{i,t}$, $i\in  [\arms]$ by \begin{equation*}
\cC_t
\triangleq[-r,r]^\arms\cap 
\pa{\vmean{t-1}+\sset{\bdelta\in \R^\arms,\norm{\pa{g_{i,t}(\delta_i)}_i}_p\leq1}},\label{def:C_t}
\end{equation*} where $g_{i,t}=0$ if $\counter{i}{t-1}=0$ and, otherwise, is convex, strictly decreasing on $[-r-\mean{i,t-1},0]$ and strictly increasing on $[0,r-\mean{i,t-1}]$ such that $g_{i,t}(0)=0$.
Typically, $r = 1$ under assumption $(iii)$ and $r=\infty$ otherwise.  
 Table~\ref{table:algos} lists  variants of Algorithm~\ref{algo:standard}, with the corresponding reward class under which they can be used.  Each of these algorithms is matching the lower bound corresponding to the respective reward class considered in Table~\ref{table:lower}, i.e., $R_T$ is a `big $\cO$' of the lower bound, up to a polylogarithmic factor in $m$ \citep{Degenne2016}.
Notice that \textsc{ThompsonSampling} is not an instance of Algorithm~\ref{algo:standard}. However, we are not aware of any tight analysis: The one 
by \citet{Wang2018} matches the lower bound $\arms m \log(T)/\Delta$, but only for mutually independent rewards, where the lower bound is $\arms \log(T)/\Delta$.
 The regret upper bound  of algorithms in Table~\ref{table:lower} with $p=1$ have an additive constant term w.r.t.\,$T$ but exponential in $\arms$, which can be replaced with a different analysis to get either: \begin{itemize} 
 \item an exponential term in $m$ plus a term of order $1/\Delta^2$,\item a term of order $1/\Delta^2$ --- by changing $\log(t)+m$ to 
  $\log(t)+m\log\log(t)$ in the algorithm, \item or can be removed --- by changing $\log(t)+m
  $ to $\log(t)+\arms\log\log(t)$ in the algorithm.\end{itemize} 
 
 On  one hand, the arbitrary correlated case can be considered as solved, since  
 the matching lower bound algorithm \mbox{\textsc{CUCB}} \citep{kveton2015tight} is efficient.\footnote{In Theorem~\ref{thm:submodconf}, we recover that $A_t$ is computed by Algorithm~\ref{algo:standard} by solving a linear optimization
problem.} On the other hand, considering the reward class given by the first line of Table~\ref{table:lower}, the known algorithms that match the lower bound are inefficient.\footnote{$\actions$ may have up to $2^\arms$ elements.}
We further discuss the efficiency 
of Algorithm~\ref{algo:standard} in the following subsection.

\subsection{Submodular Maximization}
In Algorithm~\ref{algo:standard},  only $A_t$ needs to be computed. It is a maximizer over $\actions$ of the set function 
 \begin{align}
 \begin{array}{ccc}
 \cP([\arms]) & \to & \R \hfill \\
   A & \mapsto & \max_{\bmu\in \cC_t} \be_A\transpose \bmu~. \\
\end{array}
  \label{A_t}
 \end{align}
 We can easily evaluate the function~\eqref{A_t} above for some set $A\in \cP([\arms])$, since it only requires solving a linear optimization problem on the convex\footnote{$\cC_t$ is convex since functions $g_{i,t}$ are convex.} set $\cC_t$. In Proposition~\ref{prop:close}, we show that  in some cases, the evaluation can be even simpler. However, maximizing~the function~\eqref{A_t} over a combinatorial set $\actions$ is not straightforward. 
Before studying this function more closely, Definition~\ref{def:setfct} recalls some well-known properties that can be satisfied by a set function $F:\cP([\arms])\to \R$. 

\begin{definition}\label{def:setfct}A set function $F$ is:
\begin{itemize}
\vspace{-0.1in}
\item{normalized, if}
  $F(\emptyset)=0,$
    \item{linear (or modular) if}
  $F(A)= \be_A\transpose \bdelta+b$, for some $\bdelta\in \R^\arms,~b\in \R,$
 \item{non-decreasing if} $F(A)\leq F(B)~\forall~A\subset B\subset [\arms],$ 
  \item{submodular if} for all $A,B\subset [\arms]$ ,$$F(A\cup B)+F(A\cap B)\leq F(A)+F(B).$$
Equivalently, $F$ is submodular if for all $A\subset B\subset [\arms]$, and $i\notin B$, 
$F(A\cup \sset{i})-F(A)\geq F(B\cup \sset{i})-F(B).$
 \end{itemize}
\end{definition}
 
The function~\eqref{A_t} is clearly normalized, and it can be decomposed into two set functions in the following way, \begin{align*}\forall A\subset [\arms],~ \max_{\bmu\in \cC_t} \be_A\transpose \bmu={\be_A\transpose\vmean{t-1}+ \max_{\bdelta\in \cC_t-\vmean{t-1}}\be_A\transpose\bdelta}.\end{align*} 
The linear part $A\mapsto\be_A\transpose\vmean{t-1}$ is efficiently maximized alone, we thus focus on the other part, $A\mapsto\max_{\bdelta\in \cC_t-\vmean{t-1}}\be_A\transpose\bdelta $, usually called an \emph{exploration bonus}. It aims to compensate for the negative selection bias of the first term. We define 
\begin{align*}
\cC_t^+&
\triangleq\![-r,r]^\arms\cap 
\pa{\vmean{t-1}+\sset{\bdelta\in \R^\arms_+,\norm{\pa{g_{i,t}(\delta_i)}_i}_p\leq1}}
\\&=
\cC_t\cap\sset{\bmu\in \R^\arms,~\bmu\geq \vmean{t-1}}
\end{align*}
and rewrite $A\mapsto\max_{\bdelta\in \cC_t-\vmean{t-1}}\be_A\transpose\bdelta $ through Lemma~\ref{Ct}. 
\begin{lemma}
For all $A\in \cP([\arms])$, $\max_{\bdelta\in \cC_t-\vmean{t-1}}\be_A\transpose\bdelta=\max_{\bdelta\in \cC_t^+-\vmean{t-1}}\be_A\transpose\bdelta.$
\label{Ct}
\end{lemma}
The lemma holds as $\sset{\pa{\delta_i^+}_i,~\bdelta\in \cC_t-\vmean{t-1}}\subset\cC_t-\vmean{t-1}.$ As a corollary, this set function is non-negative, and non-decreasing.
  It can be written in closed form under additional assumptions, see Proposition~\ref{prop:close} and Example~\ref{ex:close}.

\begin{proposition}\label{prop:close}
Let $A\in \cP([\arms])$, $t\in  \N^\star$, $p=1$. Assume  that for all $ i\in A$, $g_{i,t}$ has a strictly increasing, continuous derivative $g'_{i,t}$ defined on $[0,r-\mean{i,t-1}]$. For $i\in A$, let
\[
f_i(\lambda) \triangleq \left\{
    \begin{array}{ll}
        g_{i,t}'^{-1}(1/\lambda) & \text{if }1/\lambda< g_{i,t}'(r-\mean{i,t-1}),\\
        r-\mean{i,t-1} & \mbox{otherwise,}
    \end{array}
\right.
\]
defined for $\lambda\geq 0$. Then, the smallest  $\lambda^\star$ satisfying
\[\be_{A}\transpose \pa{g_{i,t}\pa{f_i(\lambda^\star)}}_i\leq 1\]
is such that
$$\pa{\delta_i^\star}_i\triangleq\pa{\II{i\in A}f_i(\lambda^\star)}_i\in \argmax_{\bdelta\in \cC_t^+-\vmean{t-1}}\be_A\transpose\bdelta.$$
\end{proposition}
The proof of Proposition~\ref{prop:close} can be found in Appendix~\ref{app:close}.
An important use-case example of Proposition~\ref{prop:close} is the following
\begin{example}Let $A\in \cP([\arms])$, $t\in  \N^\star$.
If for all $i\in [\arms]$, $g_{i,t}=\pa{\cdot}^2\alpha_{i,t}$ for some $\alpha_{i,t}>0$, and $r=\infty,p=1$, then
\[\max_{\bdelta\in \cC_t^+-\vmean{t-1}}\be_A\transpose\bdelta = \sqrt{\be_A\transpose\pa{\frac{1}{\alpha_{i,t}}}_i}\!.\]
Indeed, since the maximizer $\bdelta^\star$ lies at the boundary, $\max_{\bdelta\in \cC_t^+-\vmean{t-1}}\be_A\transpose\bdelta =\max_{\bdelta\in \R_+^\arms,~\sum_i \alpha_{i,t}\delta_i^2=1}\be_A\transpose\bdelta$, and from the first-order optimality condition we deduce that $\be_A=2\lambda^\star\pa{\alpha_{i,t}\delta^\star_i}_i$, i.e., $\delta^\star_i=\II{i\in A}/2\lambda^\star\alpha_{i,t}$, where $\lambda^\star$ is necessarily $\frac{1}{2}\sqrt{\be_A\transpose\pa{1/\alpha_{i,t}}_i}$.
We thus recover the \textsc{ESCB}'s exploration bonus for $\alpha_{i,t}=\counter{i}{t-1}/\log t.$ 
 \label{ex:close}
\end{example}

\begin{remark}
The proof of Proposition~\ref{prop:close} follows the same technique as the proof of
 Theorem~4 by \citet{combes2015combinatorial} for developing the computation of the \textsc{ESCB-KL} exploration bonus. 
\end{remark}

Example~\ref{ex:close} is a specific case where the exploration bonus $A\mapsto \max_{\bdelta\in \cC_t^+-\vmean{t-1}}\be_A\transpose\bdelta$ has a
particularly simple form: It is the square root of a non-decreasing linear set function. Such a set function is known to be submodular \citep{Stobbe2010}. This interesting property helps for maximizing the function~\eqref{A_t}. In Theorem~\ref{thm:submodconf}, we prove that $A\mapsto \max_{\bdelta\in \cC_t^+-\vmean{t-1}}\be_A\transpose\bdelta$ is in fact always submodular.
\newpage

\begin{theorem} The following two properties hold.
\begin{itemize}
 \itemsep0em
\item For $p=\infty$, $A\mapsto\max_{\bdelta\in \cC_t^+-\vmean{t-1}}\be_A\transpose\bdelta$ is linear. 
\item For $p=1$, 
$A\mapsto\max_{\bdelta\in \cC_t^+-\vmean{t-1}}\be_A\transpose\bdelta$ is submodular.
\end{itemize}
\label{thm:submodconf}
\end{theorem}
The proof is deferred to Appendix~\ref{app:submodconf} and uses a result on \emph{polymatroids} by \citet{He2012}. 
Theorem~\ref{thm:submodconf} first implies the efficiency of any variant of Algorithm~\ref{algo:standard} with $p=\infty$, since it reduces to optimizing a linear set function over~$\actions$. 
Theorem~\ref{thm:submodconf} also yields that when the reward class is strengthen to target the tighter lower bound ${{\arms\log(T)}/{\Delta}},$  Algorithm~\ref{algo:standard} reduces to maximizing a submodular set function over $\actions$ (the sum of a linear and a submodular function is submodular).
Submodular maximization problems have been applied in machine learning before (see e.g., \citealp{Krause2011,bach2011learning}),
however, maximizing a submodular function~$F$, even for $\actions=\sset{A, \abs{A}\leq m}$ and $F$ non-decreasing, is NP-Hard in general \citep{Schrijver2008}, with an approximation factor of $1+{1}/\pa{e-1}$ by the \textsc{Greedy} algorithm 
\citep{Nemhauser1978}.
This is problematic as the typical analysis is based on controlling with high probability the error $\Delta\pa{A_t}$ at round $t$ by the quantity  $2\max_{\bdelta\in \cC_t-\vmean{t-1}}\be_{A_t}\transpose\pa{\abs{\delta_i}}_i$. 
More precisely, since $\bmu^\star$ belongs with high probability to the confidence region $\cC_t$, $\bmu^\star-\vmean{t-1}$ belongs with high probability to $\cC_t-\vmean{t-1}$. Under this event, and for $\kappa\geq 1$, a $\kappa-$approximation algorithm for maximizing the function~\eqref{A_t} would only guarantee the following:   
\begin{align}\nonumber\Delta\pa{A_t}=~& \pa{\be_{A^\star}-\be_{A_t}}\transpose\bmu^\star\\\leq~&\max_{\bmu\in \cC_t}\be_{A^\star}\transpose\bmu - \be_{A_t}\transpose\bmu^\star\nonumber\\\nonumber
=~&\max_{\bmu\in \cC_t^+}\be_{A^\star}\transpose\bmu - \be_{A_t}\transpose\bmu^\star\\\label{eq:app}
\leq~&\kappa\max_{\bmu\in \cC_t^+}\be_{A_t}\transpose\bmu - \be_{A_t}\transpose\bmu^\star\\\nonumber
=~&\kappa\max_{\bdelta\in \cC_t^+-\vmean{t-1}}\be_{A_t}\transpose\bdelta \\\label{eq:appn}&+ \be_{A_t}\transpose\pa{ \vmean{t-1}-\bmu^\star}+(\kappa-1)\be_{A_t}\transpose\vmean{t-1}
\\\leq~&(\kappa+1)\max_{\bdelta\in \cC_t-\vmean{t-1}}\be_{A_t}\transpose\pa{\abs{\delta_i}}_i\nonumber+(\kappa-1)\be_{A_t}\transpose\vmean{t-1}.
\end{align}
If $\kappa\neq 1$, the term $(\kappa-1)\be_{A_t}\transpose\vmean{t-1}$ gives linear regret bounds. 
In the next section, with a stronger assumption on $\actions$ (but for which submodular maximization is still NP-Hard), we show that both parts of the objective can have different approximation factors.
More precisely, we show how to approximate the linear part with factor $1$, and the submodular part with a constant factor $\kappa >1$. Then, \eqref{eq:app} can be replaced by
$$\kappa\cdot\max_{\bmu\in \cC_t^+-\vmean{t-1}}\be_{A_t}\transpose\bmu+1\cdot\be_{A_t}\transpose\vmean{t-1}  - \be_{A_t}\transpose\bmu^\star.$$
Therefore, in \eqref{eq:appn}, the extra $(\kappa-1)\be_{A_t}\transpose\vmean{t-1}$ term is removed.

\section{Efficient Algorithms for Matroid Constraints}

In this section, we will consider additional structure on $\cA$, using the notion of matroid, recalled below.

\begin{definition}
A matroid is a pair $([\arms],\cI)$, where $\cI$ is a family of subsets of $[\arms]$, called the independent sets, with the following properties:
\label{def:matroid}
\begin{itemize}
\itemsep0em
 \item The empty set is independent, i.e., $\emptyset\in \cI$.
 \item Every subset of an independent set is independent, i.e., for all $ A\in \cI,$ if $A'\subset A$, then $A'\in\cI$.
 \item If $A$ and $B$ are two independent sets, and $\abs{A}>\abs{B}$, then there exists $x\in A\backslash B$ such that $B\cup\sset{x}\in \cI$.
\end{itemize}
\end{definition}
Matroids generalize the notion of linear independence.
A~maximal (for the inclusion) independent
set is called \emph{basis} and all bases have the same cardinality $m$, which is called the \emph{rank} of the matroid \citep{whitney35abstract}. Many combinatorial problems such as building a spanning tree for network 
routing \citep{Oliveira2005} can be expressed as a linear optimization on a matroid (see \citealp{Edmonds1965} or \citealp{Perfect1968}, for other examples). 
 \\[0.025in]
Let $\cI\in \cP([\arms])$ be such that $([\arms],\cI)$ forms a matroid. Let $\cB\subset \cI$
be the set of bases of the matroid $([\arms],\cI)$. 
In the following, we may assume that $\actions$ is either $\cI$ or $\cB$. 
We also assume that an independence oracle is available, i.e., given an input $A\subset[\arms]$, it returns $\textsc{true}$ if $A\in \cI$ and $\textsc{false}$ otherwise.
Maximizing a linear set function $L$ on $\actions$ is efficient, and it can be done as follows \citep{edmonds71matroids}:
Let $\sigma$ be a permutation of $[\arms]$ and $j$ an integer such that $j=m$ in case $\actions=\cB$ and otherwise, $j$ satisfies   \[\ell_1\geq\dots\geq \ell_j\geq 0\geq \ell_{j+1}\geq \dots \geq \ell_\arms,\]
where $\ell_i=L(\sset{\sigma(i)})~\forall i\in [\arms]$.
The optimal $S$ is built greedily 
starting from $S=\emptyset,$ and 
for $i\in [j],$ $\sigma(i)$ is added to $S$
only if $S\cup\sset{\sigma(i)}\in \cI$ .

Matroid bandits with $\actions=\cB$ has been studied by \citet{kveton2014matroid,Talebi2016}. In this case, the two lower bounds in Table~\ref{table:lower} coincide to $\Omega\pa{\arms\log(T)/\Delta}$, and \textsc{CUCB} reaches it, with the following gap-free upper bound: $R_T\pa{\textsc{CUCB}}=\cO\pa{\sqrt{\arms m T\log T}}$.
Assuming sub-Gaussian rewards to use any Algorithm 
of Table~\ref{table:algos} with $p=1$ would tighten \citep{Degenne2016} this gap-free upper bound to $\cO\pa{\sqrt{\arms\log^2 m T\log T}}$. 
Notice, due to the $\sqrt{\log T}$ factor, this does not contradict the $\Omega\pa{\sqrt{\arms m T}}$ gap-free lower bound for multi-play bandits.

In the rest of this section, we provide efficient approximation routines to maximize the function~\eqref{A_t} on $\actions=\cI$ and $\cB$ without having the extra undesirable term $(\kappa-1)\be_{A_t}\transpose\vmean{t-1}$, that a standard $\kappa-$approximation algorithm would suffer. Therefore, using these routines in Algorithm~\ref{algo:standard} do not alter its regret upper bound.

Let
$L$ be a normalized, linear set function, that will correspond to the linear part $A\mapsto\be_A\transpose\vmean{t-1}$; and let $F$ denote a normalized, non-decreasing, submodular function, that will correspond to the submodular part $A\mapsto\max_{\bdelta\in \cC_t^+-\vmean{t-1}}\be_A\transpose\bdelta$. Unless stated otherwise, we further assume that $F$ is positive (for $A\neq\emptyset$). This is a mild assumption as it holds for $A\mapsto\max_{\bdelta\in \cC_t^+-\vmean{t-1}}\be_A\transpose\bdelta$ in the unbounded case, i.e., if $(iii)$ is not assumed and $r=\infty$. If $(iii)$ is true, then adding an extra $\be_A\transpose\pa{\frac{1}{\counter{i}{t-1}^2}}_i$ term will  recover positivity and increase the regret upper bound by only an additive constant. 
In the following subsections, we will provide algorithms that efficiently outputs $S$ such that \begin{align}L(S)+\kappa F(S)\geq L(O)+F(O),\quad\forall O\in \actions ,\label{approx}\end{align}
with some appropriate approximation factor $\kappa \geq 1$.
It is possible to efficiently output $S_1$ and $S_2$ such that we get  $L(S_1)\geq L(O_1)$ and $\kappa F(S_2)\geq F(O_2)$ for any $O_1,O_2\in \actions$. Although we can take $O_1=O_2$, $S_1$ and~$S_2$ are not necessarily equal, and \eqref{approx} is not straightforward. The last subsection is we apply this approach to \emph{budgeted matroid semi-bandits}.   
\subsection{Local Search Algorithm}
In this subsection, we assume that $\actions=\cI$. In Algorithm~\ref{algo:local_search}, we provide a specific  instance of \textsc{LocalSearch} 
that we tailored to our needs to  approximately maximize \mbox{$L+F$}. 

It starts from the greedy solution $S_{\rm init}\in \argmax_{A\in \cI}L(A)$. 
Then, Algorithm~\ref{algo:local_search} repeatedly tries three basic operations in order to improve the current solution. 
Since every $S\in \actions$ can potentially be visited, only \emph{significative} improvements are considered, i.e., improvements greater than $\frac{\varepsilon}{m} F(S)$ for some input parameter $\varepsilon>0$. The smaller $\epsilon$ is, the higher complexity will be. Notice the improvement threshold $\frac{\varepsilon}{m} F(S)$ does not depend on $L$. In fact, this is crucial to ensure that the approximation factor of~$L$ is~$1$. However, this can increase the time complexity. To prevent this increase, the second essential ingredient is the initialization, where only $L$ plays a role.

 In Theorem~\ref{thm:local_search}, we state the approximation guarantees for Algorithm~\ref{algo:local_search} and its time complexity. 
 The proof of Theorem~\ref{thm:local_search} is in Appendix~\ref{app:local_search}.
For $\cC_t$ given by any algorithm of Table~\ref{table:algos}, $F(A)=\max_{\bdelta\in \cC_t^+-\vmean{t-1}}\be_A\transpose\bdelta$, and $\varepsilon=1$, the time complexity is bounded by $\cO\pa{m^2\arms \log\pa{m t}}$, and is thus efficient. 
Another algorithm enjoying an improved time complexity is provided in the next subsection, in the case where $\actions=\cB$.
\begin{theorem}
\label{thm:local_search}
 Algorithm~\ref{algo:local_search} outputs $S\in \cI$ such that
 $$L(S)+2\pa{1+\varepsilon}F(S)\geq L(O)+F(O),\quad \forall O\in \cI.$$
 Its complexity is $\cO\pa{m\arms\log\pa{\frac{\max_{A\in \cI}F(A)}{F(S_0)}}/{\log\pa{1+\frac{\varepsilon}{m}}}}.$
\end{theorem}
Theorem~\ref{thm:local_search} gives  a parameter $\kappa$ arbitrary close to~$2$ in \eqref{approx}.\footnote{We could design a different version of Algorithm~\ref{algo:local_search},  based on \textsc{Non-ObliviousLocalSearch} \citep{Filmus2012} , in order to get $\kappa$ arbitrary close to $1+{1}/\pa{e-1}\CommaBin$ but with a much worst time complexity. Actually,
\citet{Sviridenko2013} proposed such an approach,   
with an approximation factor for $L$ arbitrary close to $1$, but not equal, so we would get back the undesirable term, which would require a complexity polynomial  in $t$ to control. }
\begin{algorithm}[t]
\caption{\textsc{LocalSearch} for maximizing $L+F$ on $\cI$.} \label{algo:local_search}
\begin{algorithmic}
\STATE \textbf{Input}: $L$, $F$, $\cI$, $m$, $\varepsilon>0$.
\STATE \textbf{Initialization}: $S_{\rm init}\in \argmax_{A\in\cI}L(A).$
\IF{$S_{\rm init}= \emptyset$}
\IF{$\exists \sset{x}\in \cI$ such that $(L+F)(\sset{x})>0$}
\STATE $S_0\in \argmax_{\sset{x}\in \cI,~\pa{L+F}\pa{\sset{x}}>0}L\pa{\sset{x}}.$
\ELSE 
\STATE \textbf{Output} $\emptyset$
\ENDIF
\ELSE 
\STATE $S_0\leftarrow S_{\rm init}$
\ENDIF
\STATE $S\leftarrow S_0.$
\STATE Repeatedly perform one
of the following local improvements \textbf{while} possible:
 \\[-0.555in]
\begin{itemize}
\itemsep-0.5em
 \item \textbf{Delete an element:} \\
 \textbf{if}{ $\exists x\in S$ such that\\ $\pa{L+F}\pa{S\backslash\sset{x}}>\pa{L+F}\pa{S} +\frac{\varepsilon}{m} F(S),$}
\\\textbf{then} $S\leftarrow S\backslash\sset{x}.$\\\textbf{end if}
  \item \textbf{Add an element:} \\
 \textbf{if} {$\exists y\in [\arms]\backslash S,~S\cup\sset{y}\in \cI,$ such that \\$\pa{L+F}\pa{S\cup\sset{y}}>\pa{L+F}\pa{S} +\frac{\varepsilon}{m} F(S),$}\\
 \textbf{then} $S\leftarrow S\cup\sset{y}.$\\\textbf{end if}
 \item \textbf{Swap
a pair of elements:} \\
 \textbf{if} {$\exists (x,y)\in S\times[\arms]\backslash S,~S\backslash\sset{x}\cup\sset{y}\in \cI,$ such that 
 $\pa{L+F}\pa{S\backslash\sset{x}\cup\sset{y}}>\pa{L+F}\pa{S} +\frac{\varepsilon}{m} F(S)$}
 \textbf{then} $S\leftarrow S\backslash\sset{x}\cup\sset{y}$
 \textbf{end if}
 \\[-0.055in]
\end{itemize}
\STATE \textbf{end while}
\\[-0.025in]
\STATE \textbf{Output}: $S$.
\end{algorithmic}
\end{algorithm}

\subsection{Greedy Algorithm}
In this section, we assume that $\actions=\cB$. This  situation happens, for instance, under a non-negativity assumption on $L$, i.e., if we consider non-negative rewards $X_i$.    
We show that the standard \textsc{Greedy} algorithm (Algorithm~\ref{algo:greedy}) improves over Algorithm~\ref{algo:local_search} by exploiting this extra constraint, both in terms of the running time and the approximation factor. We state the  result in Theorem~\ref{thm:greedy} and prove it in Appendix~\ref{app:greedy}. Notice that another advantage is that we do not need to assume 
$F(A)>0$ for $A\neq \emptyset$ here.
\begin{algorithm}[t]
\caption{\textsc{Greedy} for maximizing $L+F$ on $\cB$.} \label{algo:greedy}
\begin{algorithmic}
\STATE \textbf{Input}: $L$, $F$, $\cI$, $m$.
\STATE \textbf{Initialization}: $S\leftarrow \emptyset.$
\FOR{$i\in [k]$}
\STATE $x\in\argmax_{x\notin S,~S\cup\sset{x}\in \cI}\pa{L+F}(S\cup\sset{x}).$
\STATE $S\leftarrow S\cup\sset{x}.$
\ENDFOR 
\STATE \textbf{Output}: $S$.
\end{algorithmic}
\end{algorithm}

\begin{theorem}\label{thm:greedy}
Algorithm~\ref{algo:greedy} outputs $S\in \cB$ such that $$L(S)+2F(S)\geq L(O)+F(O),\quad \forall O\in \cB.$$
Its complexity is $\cO\pa{m\arms}$.
\end{theorem}
Combining the results before, we get the following theorem.
\begin{theorem} With approximation techniques, the cumulative regret 
for the combinatorial semi-bandits is bounded as
\[R_T\leq\cO\pa{\sqrt{\arms\log^2(m)T\log T}}\] with 
per-round time complexity of order $\cO(\log(mt)m^2n)$ (resp., $\cO(mn)$) for $\actions=\cI$ (resp., $\actions=\cB$).
\end{theorem}
Notice that this new bound is better by a factor $\sqrt{m}/\log m$ than the one of \citet{kveton2014matroid} in the case $\actions=\cB$.

\subsection{Budgeted Matroid Semi-Bandits}
In this subsection, we extend results of the two previous subsections to budgeted matroid semi-bandits.
In budgeted bandits with single resource and infinite horizon \citep{ding2013multi-armed,pmlr-v45-Xia15}, each arm is associated with both a reward
and  a  cost. The agent aims at maximizing the cumulative reward under a budget constraint for the cumulative costs. \citet{xia2016budgeted} studied budgeted bandits with multiple play, where a $m$-subset $A$ of arms is selected at each round. 
An optimal (up to a constant term) offline algorithm chooses the same action $A^\star$ within each round, where $A^\star$ is the minimizer of the ratio ``expected cost paid choosing $A$'' over ``expected reward gained choosing $A$''.
In the setting of \citet{xia2016budgeted}, the agent observes the \emph{partial} random cost and reward of each arm in $A$ (i.e., semi-bandit feedback), pays the sum of partial costs of $A$ and gains the sum of partial rewards of $A$. $A^\star$ can be computed efficiently, and a \citet{xia2016budgeted} give an algorithm based on \textsc{CUCB}. It minimizes the ratio where the averages are replaced by UCBs.
We extend this setting to matroid constraints. 
We assume that total costs/rewards are non-negative linear set functions of the chosen action $A$.
The objective is to minimize a ratio of linear set functions. 
As previously, two kinds of constraints can be considered for the minimization: either $\actions=\cI$ or $\actions=\cB$.
Theorem~\ref{thm:submodconf} implies that an optimistic estimation of this ratio is of the form
$\frac{L_1-F_1}{L_2+F_2}\CommaBin$
where for $i\in \sset{1,2}$, $F_i$ are positive (except for $\emptyset$), normalized, non-decreasing, submodular; and $L_i$ are non-negative and linear. 
$L_1-F_1$ is a high-probability lower bound on the expected cost paid, and $L_2+F_2$ is a high-probability upper bound on the expected reward gained.
Notice that the numerator, $L_1-F_1$, can be negative, which can be an incitement to take arms with a high cost/low rewards. Therefore, we consider the minimization of the surrogate
$\pa{\frac{L_1-F_1}{L_2+F_2}}^{\!\!+}\!\!.$
Indeed, $\pa{L_1-F_1}/\pa{L_2+F_2}$ is a high probability lower bound on the ratio of expectation, so by monotonicity of $x\mapsto x^+$ on $\R$, ${\pa{L_1-F_1}^+\!/\pa{L_2+F_2}}$\! is also a high-probability lower bound.
We assume $L_2$ is normalized, but not necessarily $L_1$. $L_1(\emptyset)$ can be seen as an entry price.
When $L_1$ is normalized, we assume that $\emptyset$ is not feasible.
\begin{remark}
 Notice, If $\actions=\cI$, and $L_1$ is normalized, then there is an optimal solution of the form $\sset{s}\in \cI$:
If $L_1-F_1$ is negative for some $S=\sset{s}\subset \cI$, then such $S$ is a minimizer. Otherwise, by submodularity (and thus subadditivity, since we consider normalized functions), $L_1-F_1$ is non-negative, and we have \begin{align*}\frac{L_1(S)-F_1(S)}{L_2(S)+F_2(S)}&\geq \frac{\sum_{s\in S}L_1(\sset{s})-F_1(\sset{s})}{\sum_{s\in S}L_2(\sset{s})+F_2(\sset{s})}\\&\geq \min_{s\in S}\frac{L_1(\sset{s})-F_1(\sset{s})}{L_2(\sset{s})+F_2(\sset{s})}\cdot\end{align*}
\end{remark}

This minimization problem is at least as difficult as previous submodular maximization problems, taking $L_1=1$ and $F_1=0$.  
In order to use our approximation algorithms, we consider the \emph{Lagrangian function} associated to the problem (see e.g., \citealp{fujishige2005submodular}),
$$\cL(\lambda,S)\triangleq{L_1(S)-F_1(S)} - \lambda\pa{L_2(S)+F_2(S)},$$
for $\lambda\geq 0$ and $S\subset[\arms].$
For a fixed $\lambda\geq 0,$ $-\cL(\lambda,\cdot)$ is a sum of a linear and a submodular function, and both Algorithms~\ref{algo:local_search} and~\ref{algo:greedy} can be used. However, the first step is to find $\lambda$ sufficiently close to the optimal ratio \[\lambda^\star=\min_{A\in\actions}\pa{\frac{L_1(A)-F_1(A)}{L_2(A)+F_2(A)}}^{\!\!+}\!\cdot\] 
\begin{remark}\label{rk:Lagrangian} For some $\lambda \geq 0,$
$$\min_{A\in\actions} \cL(\lambda,A)\geq 0\imp\lambda\leq \lambda^\star,$$ 
\[
  \min_{A\in\actions} \cL(\lambda,A)\leq 0\imp\left\{
     \begin{array}{c}\lambda\geq \lambda^\star,\text{ or}\\
       \min_{A\in\actions} {L_1(A)-F_1(A)}\leq 0,\\ \text{which further gives }\lambda^\star=0. 
     \end{array}
   \right.
\]
\end{remark}
From Remark~\ref{rk:Lagrangian}, if it was possible to compute $\min_{A\in\actions} \cL(\lambda,A)$ exactly, then a binary search algorithm would find $\lambda^\star$. This dichotomy method can be extended to $\kappa-$approximation algorithms by defining the \emph{approximation Lagrangian} as 
\[\cL_{\kappa}(\lambda,S)\triangleq L_1(S)-\kappa F_1(S) - \lambda\pa{L_2(S)+\kappa F_2(S)},\]
for $\lambda\geq 0$ and $S\subset[\arms].$
The idea is to use the following approximation guarantee for a $\kappa-$approximation algorithms  outputing $S$ (with objective function $-\cL$),
\[\min_{A\in\actions}\cL_{\kappa}(\lambda,A)\leq\cL_{\kappa}(\lambda,S)\leq \min_{A\in\actions} \cL(\lambda,A).\]
Thus, for a given $\lambda$, either the l.h.s is strictly negative or the r.h.s is non-negative, depending on the sign of $\cL_{\kappa}(\lambda,S).$ 

Therefore, from Remark~\ref{rk:Lagrangian}, a lower bound $\lambda_1$ on $\lambda^\star$, and an upper bound $\lambda_2$ on $\min_{A\in\actions}\pa{\frac{L_1(A)-\kappa F_1(A)}{L_2(A)+\kappa F_2(A)}}^+$ can be computed.
The detailed method is given in Algorithm~\ref{algo:ratio}.  
Notice that it takes as input some $\textsc{Algo}_{\kappa}$, that can be either Algorithm~\ref{algo:local_search} or Algorithm~\ref{algo:greedy}, depending on the assumption on the constraint (either $\actions=\cI$ or $\actions=\cB$). We denote the output as $\textsc{Algo}_{\kappa}(L+F)$, for some linear set function $L$ and some submodular set function $F$, for maximizing the objective $L+F$ on $\actions$, so that $S=\textsc{Algo}_{\kappa}(L+F)$ satisfies $L(S)+\kappa F(S)\geq \max_{A\in\actions} L(A)+F(A)$, i.e., $\kappa=2(1+\varepsilon)$ if $\textsc{Algo}_{\kappa}= $ Algorithm~\ref{algo:local_search}, $\actions=\cI$ and $\kappa=2$ if $\textsc{Algo}_{\kappa}= $ Algorithm~\ref{algo:greedy}, $\actions=\cB$.
\begin{algorithm}[t]
\caption{Binary search for minimizing the ratio ${\pa{L_1-F_1}^+/\pa{L_2+F_2}}$.} \label{algo:ratio}
\begin{algorithmic}
\STATE \textbf{Input}: $L_1,L_2$, $F_1,F_2$,  $\textsc{Algo}_{\kappa}$, $\eta >0$.
\STATE $\delta\leftarrow \frac{\eta \min_{\sset{s}\in\actions} F_1(\sset{s})}{L_2(B)+\kappa^2 F_2(B)}$ with $B=\textsc{Algo}_{\kappa}(L_2+\kappa F_2).$
\STATE $A\leftarrow A_0\in  \actions\backslash \sset{\emptyset}$ arbitrary.
\IF{ $\cL_\kappa(0,A)> 0$}
\STATE $\lambda_1\leftarrow0,\quad \lambda_2\leftarrow\frac{L_1(A)-F_1(A)}{L_2(A)+F_2(A)}\cdot$
\WHILE{${\lambda_2-\lambda_1}\geq \delta$}
\STATE $\lambda\leftarrow\frac{\lambda_1+\lambda_2}{2}\cdot$  
\STATE $S\leftarrow\textsc{Algo}_{\kappa}(-\cL(\lambda,\cdot)).$
\IF{$\cL_\kappa(\lambda,S)\geq 0$}
\STATE $\lambda_1\leftarrow\lambda.$ 
\ELSE
\STATE $\lambda_2\leftarrow\lambda.$
\STATE $A\leftarrow S.$
\ENDIF
\ENDWHILE
\ENDIF
\STATE \textbf{Output}: $A$.
\end{algorithmic}
\end{algorithm}
In Theorem~\ref{thm:ratio}, we state the result for the output of Algorithm~\ref{algo:ratio} and prove it in Appendix~\ref{app:ratio}. 
\begin{theorem}
\label{thm:ratio}
Algorithm~\ref{algo:ratio} outputs $A$ such that
$$\pa{\frac{L_1(A)-(\kappa+\eta) F_1(A)}{L_2(A)+\kappa F_2(A)}}^{\!\!+}\!\!\leq\lambda^\star,$$
where $\lambda^\star$ is the minimum of $\pa{\frac{L_1-F_1}{L_2+F_2}}^{\!\!+}$ over $\cI$ if $\textsc{Algo}_{\kappa}= $ Algorithm~\ref{algo:local_search}, and over $\cB$ if $\textsc{Algo}_{\kappa}= $ Algorithm~\ref{algo:greedy}. 
For $\cC_t$ given by any algorithm of Table~\ref{table:algos}, $F(A)=\max_{\bdelta\in \cC_t^+-\vmean{t-1}}\be_A\transpose\bdelta$,
the complexity is of order $\log(m t/\eta)$ times the complexity of $\textsc{Algo}_{\kappa}.$
\end{theorem}
\section{Experiments}
We provide experiments for a \emph{graphic matroid}, on a five nodes complete graph, as did \citet{combes2015combinatorial}. We thus have $\arms=10$, $m=4$. We consider two experiments. In the first one we use 
$\actions=\cB$, $\mu^\star_i=1+\Delta\II{i\leq m},$ for all $i\in [\arms],$ and in the second, $\actions=\cI$, where we set $\mu^\star_i=\Delta(2\II{i\leq m-1}-1),$ $\forall i\in [\arms]$. We take $\Delta=0.1$, with rewards drawn from independent unit variance Gaussian distributions.  
Figure~\ref{exp:1} illustrates the comparison between \textsc{CUCB} and our implementations of $\textsc{ESCB}$ \citep{combes2015combinatorial} using Algorithm~\ref{algo:greedy} (left) and \ref{algo:local_search} (right, with $\varepsilon=0.1$), showing the behavior of the regret vs.\,time horizon. We observe that although we are approximating the confidence region within a factor at least $2$ (and thus force the exploration), our efficient implementation outperforms $\textsc{CUCB}$. Indeed, we take advantage (gaining a factor $\sqrt{m}/\log m$) of the previously inefficient algorithm (that we made efficient) to use reward independence, which the more conservative $\textsc{CUCB}$ is not able to. The latter algorithm has still a better per round-time complexity of $\cO\pa{\arms\log m}$ and may be more practical on larger instances. 

\begin{figure}[H]
\vskip -0.12in
\centering

\resizebox{\columnwidth}{!}
{\includegraphics{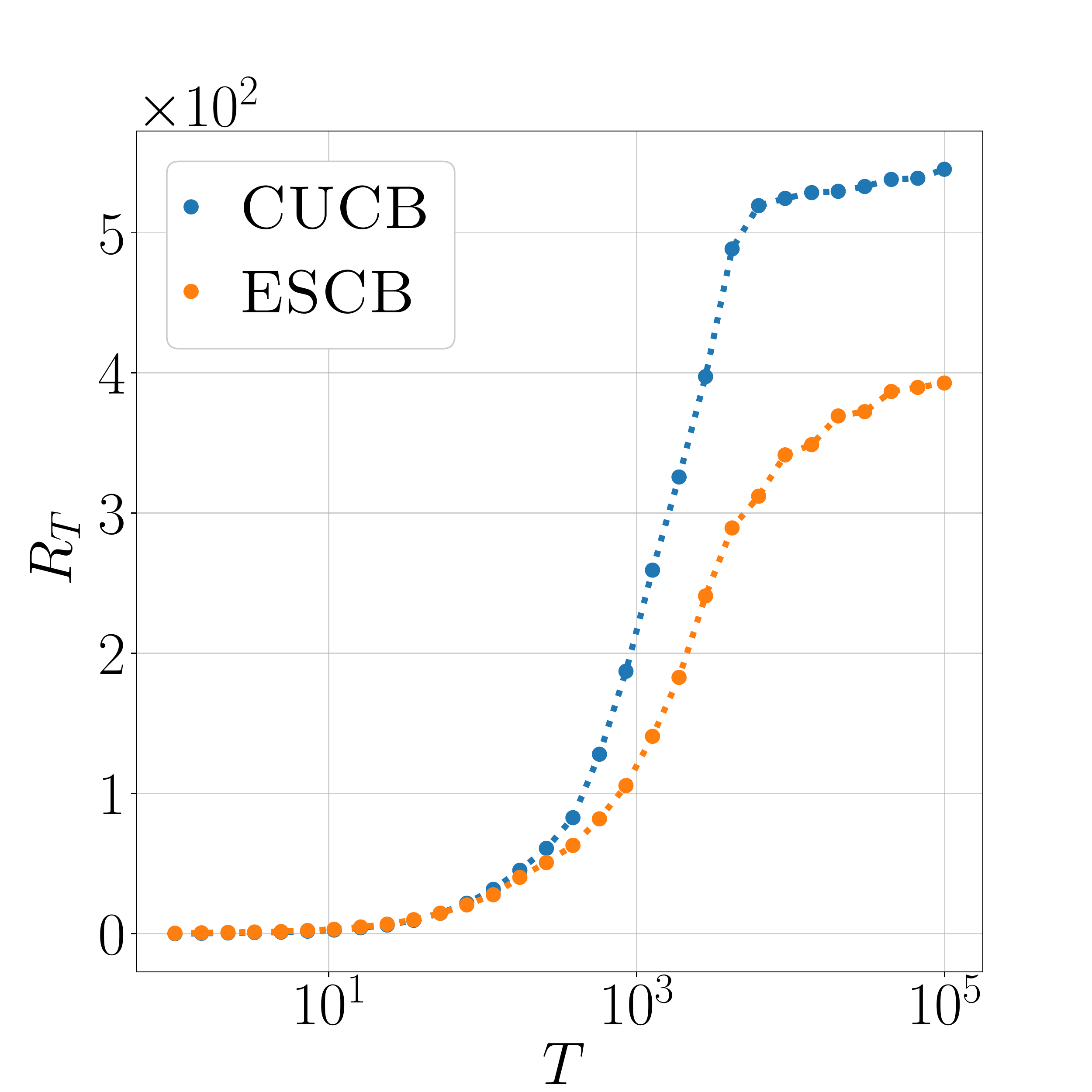}\includegraphics{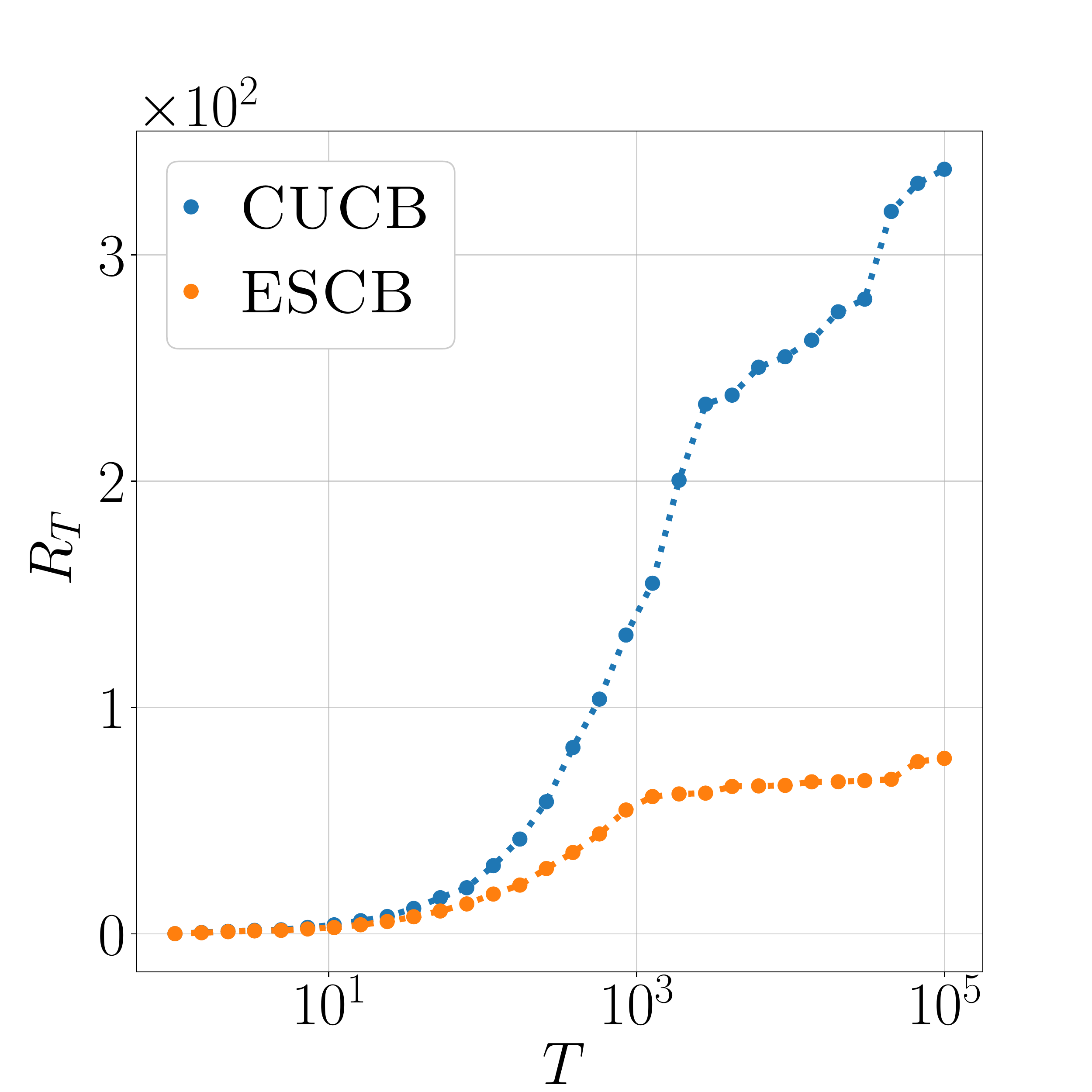}
}
\vskip -0.1in
\caption{Cumulative regret for the minimum spanning tree setting in up to $10^5$
rounds, averaged over 100 independent simulations. \textbf{Left}: for $\actions=\cB$. \textbf{Right}: for $\actions=\cI$.}
\vskip -0.1in
\label{exp:1}
\end{figure}

\section{Discussion}
In this paper, we gave several approximation schemes 
for the confidence regions and applied them to combinatorial semi-bandits
with matroid constraints and their budgeted version. 
We believe our approximation methods can be extended to approximation regret for
non-linear objective functions (e.g., for influence maximization, \citealp{Wang2018}), if the maximization algorithm
keeps the same approximation factor for the objective, either with or without the bonus.
\clearpage

\paragraph{Acknowledgments}
Vianney Perchet has benefited from the support of the ANR (grant n.ANR-13-JS01-0004-01), of the FMJH Program Gaspard Monge in optimization and operations research (supported in part by EDF), from the Labex LMH and from the CNRS through the PEPS program. The research presented was also supported by European CHIST-ERA project DELTA, French Ministry of
Higher Education and Research, Nord-Pas-de-Calais Regional Council,
Inria and Otto-von-Guericke-Universit\"at Magdeburg associated-team north-European project Allocate, and French National Research Agency project BoB (grant n.ANR-16-CE23-0003), 
FMJH Program PGMO with the support to this program from Criteo.

\bibliography{library,example}
\bibliographystyle{icml2019}

\newpage
\onecolumn
\appendix
\section{Proof of Proposition~\ref{prop:close}}
\label{app:close}
\begin{proof} It suffices to maximize on the coordinates of $\bdelta$ belonging to $A$ (the others being zero).
For all $i\in A$, we let \begin{align*}\eta_i^\star&\triangleq\pa{1-\lambda^\star g_{i,t}'(r-\mean{i,t-1})}\II{\delta^\star_i= r-\mean{i,t-1}}\\                                \gamma_i^\star&\triangleq\pa{\lambda^\star g_{i,t}'(0)-1}\II{\delta^\star_i= 0}=-\II{\delta^\star_i= 0}.
\end{align*}
For all $i\in A$, the function $f_i$ is continuous, non-increasing on $\R_+$, hence so is $\lambda\mapsto \be_{A}\transpose \pa{g_{i,t}\pa{f_i(\lambda)}}_i$. 
If $\be_{A}\transpose \pa{g_{i,t}\pa{f_i(\lambda^\star)}}_i< 1$, then necessarily $\lambda^\star=0.$  Thus, the following KKT conditions are satisfied:
 \begin{align*} \lambda^\star\pa{\sum_{i\in A}g_{i,t}(\delta_i^\star)-1}&=0,\text{ and}\\\forall i\in A,~\lambda^\star g_{i,t}'(\delta_i^\star)+\eta^\star_i-\gamma_i^\star&=1,
 \\
  \eta_i^\star (\delta_i^\star-r+\mean{i,t-1})&=0,
 \\ 
 -\gamma_i^\star \delta_i^\star&=0,
 \end{align*}
which concludes the proof by the convexity of the constraints and the objective function.
\end{proof}

\section{Proof of Theorem~\ref{thm:submodconf}}
\label{app:submodconf}
\begin{proof}
Let $t\in\N^\star $. We consider here the restriction of $g_{i,t}$ to $[0,r-\mean{i,t-1}]$, that we still denote  as $g_{i,t}$.
Notice that for all $i\in [\arms]$, $g_{i,t}$ is either $0$ or a bijection on $[0,r-\mean{i,t-1}]$ by assumption.
For $p=\infty$, we have that $$\max_{\bdelta\in \cC_t^+-\vmean{t-1}}\be_A\transpose\bdelta=\be_A\transpose\pa{\min\sset{g^{-1}_{i,t}(1),r-\mean{i,t-1}}\II{\counter{i}{t-1}\geq 1}+r\II{\counter{i}{t-1}=0}}_i$$
is a linear set function of $A$.
Assume now that $p=1$. To show the submodularity of $A\mapsto\max_{\bdelta\in \cC_t^+-\vmean{t-1}}\be_A\transpose\bdelta$ in this case, we will use the notion of \emph{polymatroid}.
\begin{definition}[Polymatroid]
 A polymatroid is a polytope of the forme $\sset{\bdelta'\in\R_+^\arms, \be_A\transpose\bdelta' \leq F(A),~\forall A\subset [\arms]}$, where $F$ is a non-decreasing submodular function.
\end{definition}
\begin{fact}[Theorem~3 of \citealp{He2012}]
 Let $P$ be a polymatroid, and let $h_1,\dots,h_\arms$ be concave functions. Then $A\mapsto \max_{\bdelta'\in P}\be_A\transpose\pa{h_i(\delta'_i)}_i$ is submodular.
 \label{he2012}
\end{fact}

Notice that $g_{i,t}^{-1}\pa{\sset{0}}=[0,r-\mean{i,t-1}]$ when $\counter{i}{t-1}=0$, and that $g_{i,t}^{-1}\pa{\cdot}$ is a strictly increasing concave function on $[0,g_{i,t}(r-\mean{i,t-1})]$, as the inverse function of a strictly increasing convex function when $\counter{i}{t-1}\geq 1$. So we can rewrite $\cC_t^+-\vmean{t-1}$ as an union of product sets:
\begin{align*}\cC_t^+-\vmean{t-1}&=\sset{\bdelta\in \prod_{i\in[\arms]}[0,r-\mean{i,t-1}],~\sum_{i\in [\arms]}g_{i,t}(\delta_i)\leq 1}
=\bigcup_{\overset{\bdelta'\in \prod_{i\in[\arms]}[0,g_{i,t}(r-\mean{i,t-1})],}{\sum_{i\in [\arms]}\delta_i'\leq 1}}{~{~\prod_{i\in [\arms]}g_{i,t}^{-1}\pa{\sset{\delta'_i}}}}.\end{align*}
We can thus rewrite our function as \begin{align*}\max_{\bdelta\in \cC_t^+-\vmean{t-1}}\be_A\transpose\bdelta&=\max_{\overset{\bdelta'\in \prod_{i\in[\arms]}[0,g_{i,t}(r-\mean{i,t-1})],}{\sum_{i\in [\arms]}\delta_i'\leq 1}}\be_A\transpose \pa{g_{i,t}^{-1}\pa{\delta'_i}}_i,\end{align*}
with the convention $g_{i,t}^{-1}\pa{0}=r-\mean{i,t-1}$ when $\counter{i}{t-1}=0$.

The constraints' set ${\sset{\bdelta'\in \prod_{i\in[\arms]}[0,g_{i,t}(r-\mean{i,t-1})],~\sum_{i\in [\arms]}\delta_i'\leq 1}}$ is equal to the intersection between $\prod_{i\in[\arms]}[0,g_{i,t}(r-\mean{i,t-1})] $ and the polymatroid $\sset{\bdelta'\in\R_+^\arms, \be_A\transpose\bdelta' \leq \II{A\neq \emptyset},~\forall A\subset [\arms]}$. This intersection is itself equal to the polymatroid $\sset{\bdelta'\in\R_+^\arms, \be_A\transpose\bdelta' \leq \min_{B\subset A}\sset{ \II{B\neq A}+\be_B\transpose \pa{g_{i,t}(r-\mean{i,t-1})}_i},~\forall A\subset [\arms]}$.

Thus, $\max_{\bdelta\in \cC_t^+-\vmean{t-1}}\be_A\transpose\bdelta$ is the optimal objective value on a polymatroid of a  separable  concave  function, as a function of the index set $A$. Now, using Fact~\ref{he2012}, it is submodular. \end{proof}

\begin{figure}[H]

\centering

\resizebox{0.4\columnwidth}{!}
{\includegraphics[trim=120 120 120 120,clip,width=0.1\textwidth]{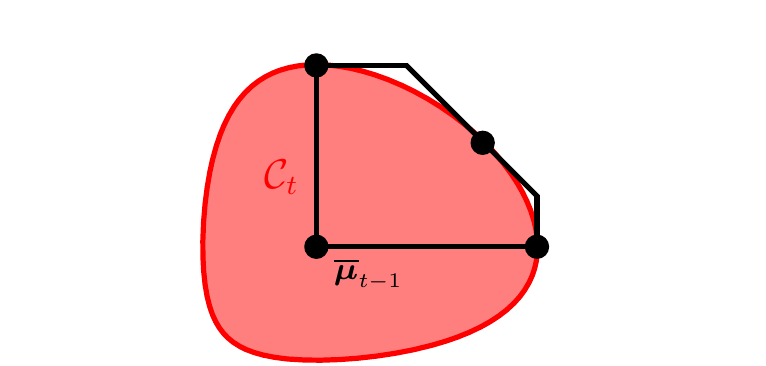}
}

\caption{Illustration of Theorem~\ref{thm:submodconf}: The confidence region $\cC_t$ and the polymatroid defined by the submodular function $A\mapsto \max_{\bmu\in\cC_t-\vmean{t-1}}\be_{A}\transpose \bmu$.}
\vskip -0.1in
\label{exp:1}
\end{figure}

\section{Proof of Theorem~\ref{thm:local_search}}
\label{app:local_search}
Before proving Theorem~\ref{thm:local_search}, we state some well known results about submodular optimization on a matroid.
\begin{proposition}\label{submod:property} 
 Let $A,B\subset [\arms]$. If $F$ is submodular, then\begin{align*}
  \sum_{b\in B\backslash A}\pa{F(B)-F(B\backslash\sset{b})}\leq F(B)-F(A\cap B),\quad
  \sum_{a\in A\backslash B}\pa{F(B\cup\sset{a})-F(B)}\geq F(A\cup B)-F(B).
 \end{align*}
\end{proposition}
\begin{proof}
Let $(b_1,\dots,b_{\abs{B\backslash A}})$ be an ordering of $B\backslash A$. Then, by submodularity of $F$,
 \begin{align*}
  \sum_{i=1}^{\abs{B\backslash A}}\pa{F(B)-F(B\backslash\sset{b_i})}&\leq \sum_{i=1}^{\abs{B\backslash A}}\pa{F(B\backslash\sset{b_1,\dots,b_{i-1}})-F(B\backslash\sset{b_1,\dots,b_{i}})}
  =
  F(B)-F(A\cap B).
 \end{align*}
 In the same way, let $(a_1,\dots,a_{\abs{A\backslash B}})$ be an ordering of $A\backslash B$.
Then, by submodularity of $F$,
 \begin{align*}
  \sum_{i=1}^{\abs{A\backslash B}}\pa{F(B\cup\sset{a_i})-F(B)}&\geq \sum_{i=1}^{\abs{A\backslash B}}\pa{F(B\cup\sset{a_1,\dots,a_{i}})-F(B\cup\sset{a_1,\dots,a_{i-1}})}
  =
  F(A\cup B)-F(B).
 \end{align*}
 \end{proof}
\begin{fact}[Theorem~1 of \citealp{Lee2010}]\label{lee_mapping}
 Let $A,B\in \actions.$ Then, there exists a mapping $\alpha : B\backslash A \to A\backslash B \cup \sset{\emptyset}$ such that 
 \vskip -0.1in
\begin{itemize}
 \item $\forall b\in B\backslash A, ~A\backslash\sset{\alpha(b)}\cup b\in \actions$ 
 \item $\forall a\in A\backslash B, ~\abs{\alpha^{-1}(a)}\leq 1.$
\end{itemize}
\end{fact}

\begin{proposition}\label{2app}
Let $A,B\in \actions$. Let $F$ be a submodular function and  $\alpha: B\backslash A\to A\backslash B\cup \sset{\emptyset}$ be the mapping given in Fact~\ref{lee_mapping}. Then,
 \begin{align*}\sum_{b\in B\backslash A}\pa{F(A)-F(A\backslash\sset{\alpha(b)}\cup \sset{b})} +\sum_{a\in A\backslash B,~\alpha^{-1}(a)=\emptyset}(F(A)-F(A\backslash\sset{a})) \leq  2 F(A)-F(A\cup B)-F(A\cap B).\end{align*}
\end{proposition}
\begin{proof}
 We decompose  $\sum_{b\in B\backslash A}\pa{F(A)-F(A\backslash\sset{\alpha(b)}\cup\sset{b})}$ into sum of two terms,
$$\sum_{b\in B\backslash A}\pa{F(A)-F(A\backslash\sset{\alpha(b)})}+\sum_{b\in B\backslash A}\pa{F(A\backslash\sset{\alpha(b)})-F(A\backslash\sset{\alpha(b)}\cup\sset{b})}.$$
Remark that the first part is equal to $$\sum_{a\in \alpha\pa{B\backslash A}}\pa{F(A)-F(A\backslash\sset{a})}=\sum_{a\in A\backslash B,~\alpha^{-1}\pa{a}\neq\emptyset}\pa{F(A)-F(A\backslash\sset{a})}.$$
Thus, together with $\sum_{a\in A\backslash B,~\alpha^{-1}(a)=\emptyset}\pa{{F}(A)-{F}(A\backslash\sset{a})}$, we get 
that 
$$\sum_{b\in B\backslash A}\pa{F(A)-F(A\backslash\sset{\alpha(b)}\cup \sset{b})} +\sum_{a\in A\backslash B,~\alpha^{-1}(a)=\emptyset}(F(A)-F(A\backslash\sset{a}))$$ is equal to
$$\sum_{a\in A\backslash B}\pa{F(A)-F(A\backslash\sset{a})}+\sum_{b\in B\backslash A}\pa{F(A\backslash\sset{\alpha(b)})-F(A\backslash\sset{\alpha(b)}\cup\sset{b})}.$$
Finally, we upper bound the first term by $F(A)-F(A\cap B)$ using first inequality of Lemma~\ref{submod:property}, and the second term by $F(A)-F(A\cup B)$ using first, the submodularity of $F$ to remove $\alpha(b)$ in the summands, and then the second inequality of Lemma~\ref{submod:property}.
\end{proof}

\begin{proof}[Proof of Theorem~\ref{thm:local_search}]
The proof is divided into two parts:
\paragraph{Approximation guarantee}If Algorithm~\ref{algo:local_search} outputs $\emptyset$ before entering in the while loop, then by submodularity, for any $S\in \cI,$ $$(L+F)(S)\leq \sum_{x\in S}(L+F)\pa{\sset{x}}\leq 0.$$
Thus, $\emptyset$ is a maximizer of $L+F$.

Otherwise, the output $S$ of Algorithm~\ref{algo:local_search} satisfies the local optimality of the while loop.
We apply Proposition~\ref{2app} with $A=S$ and $B=O$ for $L$ and $F$ separately,
$$\sum_{b\in O\backslash S}\pa{L(S)-L(S\backslash\sset{\alpha(b)}\cup \sset{b})} +\sum_{a\in S\backslash O,~\alpha^{-1}(a)=\emptyset}(L(S)-L(S\backslash\sset{a})) \leq  2 L(S)-L(S\cup O)-L(S\cap O),$$
$$\sum_{b\in O\backslash S}\pa{F(S)-F(S\backslash\sset{\alpha(b)}\cup \sset{b})} +\sum_{a\in S\backslash O,~\alpha^{-1}(a)=\emptyset}(F(S)-F(S\backslash\sset{a})) \leq  2 F(S)-F(S\cup O)-F(S\cap O).$$
Then, we sum these two inequalities,
\begin{align*}\sum_{b\in O\backslash S}\pa{\pa{L+F}(S)-\pa{L+F}(S\backslash\sset{\alpha(b)}\cup \sset{b})} +&\sum_{a\in S\backslash O,~\alpha^{-1}(a)=\emptyset}(\pa{L+F}(S)-\pa{L+F}(S\backslash\sset{a}))\\ &\leq  2 \pa{L+F}(S)-\pa{L+F}(S\cup O)-\pa{L+F}(S\cap O)\\&=2 F(S)-F(S\cup O)-F(S\cap O) +L(S)-L(O),\end{align*}
where the last equality uses linearity of $L$. Since $F$ is increasing and non-negative, $F(S\cup O)+F(S\cap O)\geq F(O)$, and we get  \begin{align*}\sum_{b\in O\backslash S}\pa{\pa{L+F}(S)-\pa{L+F}(S\backslash\sset{\alpha(b)}\cup \sset{b})} +\sum_{a\in S\backslash O,~\alpha^{-1}(a)=\emptyset}(\pa{L+F}(S)-\pa{L+F}(S\backslash\sset{a}))\\ \leq  2 F(S)-F(O) +L(S)-L(O).\end{align*}
From the local optimality of $S$, the left hand term in this inequality is lower bounded by \[\sum_{b\in O\backslash S}\frac{-\varepsilon}{m} F(S)+\sum_{a\in S\backslash O,~\alpha^{-1}(a)=\emptyset}\frac{-\varepsilon}{m} F(S)\geq -2\varepsilon F(S).\] The last  statement finishes the proof for the approximation inequality.

\paragraph{Time complexity}  Computing $S_0$ has a  negligible complexity compared to the while loop. The following lemma gives a characterization of $S_0$.
\begin{lemma}  
$S_0\in\argmax\sset{L(A),~A\in \cI,~\pa{F+L}\pa{A}>0}.$
\end{lemma}

\begin{proof}From Algorithm~\ref{algo:local_search},
if $S_{\rm init}\neq \emptyset$, then $S_0=S_{\rm init}$ and $L(S_0)=\max_{A\in \cI}L(A)\geq 0.$ Thus, $F(S_0)>0$ by assumption on $F$, giving $\pa{F+L}\pa{S_0}>0,$ which ends the proof. If $S_{\rm init}= \emptyset$, then $L(S_0)=\max\sset{L\pa{\sset{x}},~\sset{x}\in \cI,~\pa{L+F}\pa{\sset{x}}>0}.$ Let $A\in\argmax\sset{L(A),~A\in \cI,~\pa{F+L}\pa{A}>0}$. $A$ is clearly non-empty, and by submodularity of $F+L$, there exists $x\in A$ such that $\pa{F+L}\pa{\sset{x}}>0.$ $L$ is non-increasing from
$S_{\rm init}= \emptyset$, so we get $L\pa{\sset{x}}\geq L\pa{A}$, which means there is a singleton $\sset{x}$ in $\argmax\sset{L(A),~A\in \cI,~\pa{F+L}\pa{A}>0}$, so $S_0\in\argmax\sset{L(A),~A\in \cI,~\pa{F+L}\pa{A}>0}$,
which finishes the proof.  
\end{proof}

From this lemma, necessarily $L(S_0)\geq L(S_\ell)$ for every iterations $\ell \geq 1$, since the sequence $\pa{L(S_\ell)+F(S_\ell)}_\ell$ is increasing, and thus $(F+L)(S_\ell)>0,~ \forall \ell \geq 1$.  
At each iteration $\ell\geq 1$, Algorithm~\ref{algo:local_search} constructs $S_\ell$ such that $$F(S_\ell)> \pa{1+\frac{\varepsilon}{m}}F(S_{\ell-1}) + L(S_{\ell-1})-L(S_{\ell}).$$  
 Thus, we must have 
\begin{align*}F(S_\ell) - \pa{1+\frac{\varepsilon}{m}}^\ell F(S_0)&\geq\sum_{j=1}^\ell \pa{1+\frac{\varepsilon}{m}}^{\ell-j} \pa{L(S_{j-1})-L(S_j)}\\
&=L(S_0)\pa{1+\frac{\varepsilon}{m}}^{\ell-1}-\frac{\varepsilon}{m}\sum_{j=1}^{\ell-1} L(S_j)\pa{1+\frac{\varepsilon}{m}}^{\ell-j-1}-L(S_\ell)\\
&\geq L(S_0)\pa{1+\frac{\varepsilon}{m}}^{\ell-1}-\frac{\varepsilon}{m}\sum_{j=1}^{\ell-1} L(S_0)\pa{1+\frac{\varepsilon}{m}}^{\ell-j-1}-L(S_0)=0,
\end{align*}
where the last inequality uses $L(S_0)\geq L(S_\ell),~ \forall \ell \geq 1$.
This gives the following upper bound on the number of iteration $\ell$: $$\ell\leq \frac{\log\pa{\frac{F(S_\ell)}{F(S_0)}}}{\log\pa{1+\frac{\varepsilon}{m}}}\leq \frac{\log\pa{\frac{\max_{A\in \actions}F(A)}{F(S_0)}}}{\log\pa{1+\frac{\varepsilon}{m}}}\cdot$$
Finally, the result follows remarking that time complexity per iteration is $\cO\pa{m\arms}$.
\end{proof}

\section{Proof of Theorem~\ref{thm:greedy}}
\label{app:greedy}
As we did in the previous section, before starting the proof of Theorem~\ref{thm:greedy}, we state some useful results. 
\begin{fact}[Brualdi's lemma]\label{brualdi_mapping}
Let $A,B\in \cB$. Then, there exists a bijection $\beta : A \to B$ such that $$\forall a\in A,~ A\backslash\sset{a}\cup\sset{\beta(a)}\in \cB.$$
Furthermore, $\beta$
is the identity on $A\cap B$.
\end{fact}\begin{proof}
A proof is given by \citet{Brualdi1969} and is also proved by \citet{Schrijver2008}, as Corollary~39.12a. 
\end{proof}

\begin{proposition}
Let $A,B\in \cB$. Let $F$ be a submodular function and  $\beta: A\to B$ be the mapping given in Fact~\ref{brualdi_mapping}. Let $a_1,\dots,a_k$ be elements of $A$, and $A_i=\sset{a_1,\dots,a_i}$. Then,
 \begin{align*}\sum_{i\in [k]}\pa{F(A_i)-F(A_{i-1}\cup\sset{\beta(a_i)})} \leq  2 F(A)-F(A\cup B)-F(\emptyset).\end{align*}\label{2app'}
\end{proposition}
\begin{proof}
We can split $\sum_{i\in [k]}\pa{F\pa{A_{i-1}\cup\sset{a_i}}-F\pa{A_{i-1}\cup\sset{\beta\pa{a_i}}}}$ into two terms,
\begin{align*}
\sum_{i=1}^k\pa{F\pa{A_{i-1}\cup\sset{a_i}}-F\pa{A_{i-1}}
}+\sum_{i=1}^k\pa{{F\pa{A_{i-1}}-F\pa{A_{i-1}\cup\sset{\beta\pa{a_i}}}}}.
\end{align*}
The first term is equal to $F(A_k)-F(\emptyset)$. Using submodularity of $F$, the second term is upper bounded by $$\sum_{i=1}^k\pa{{F\pa{A_{m}}-F\pa{A_{m}\cup\sset{\beta\pa{a_i}}}}},$$ which is upper bounded by $F(A_k)-F(A_k\cup B)$ thanks to Proposition~\ref{submod:property} and its second inequality.
\end{proof}
\begin{proof}[Proof of Theorem~\ref{thm:greedy}]
The time complexity proof is trivial. 
 Let $S_i\triangleq\sset{s_1,\dots,s_i}$ be the set maintained in Algorithm~\ref{algo:greedy} after~$i$ iterations. Instatianting Proposition~\ref{2app'} with $A_i=S_i$ and $B=O$, we have \begin{align}\label{greedysub}\sum_{i\in [k]}\pa{F(S_i)-F(S_{i-1}\cup\sset{\beta(s_i)})} \leq  2 F(S)-F(S\cup O)-F(\emptyset).\end{align}
Furthermore, we also have, by linearity of $L$, and bijectivity of $\beta$,
\begin{align}\sum_{i\in [k]}\pa{L(S_i)-L(S_{i-1}\cup\sset{\beta(s_i)})} = \sum_{i\in [k]}\pa{L(\sset{s_i})-L(\sset{\beta(s_i)}))}=L(S)-L(O).\label{greedylin}\end{align}
Thus, we can sum up \eqref{greedysub} and \eqref{greedylin} to get 
\begin{align*}\label{greedysub}\sum_{i\in [k]}\pa{\pa{L+F}(S_i)-\pa{L+F}(S_{i-1}\cup\sset{\beta(s_i)})} \leq  2F(S) -F(S\cup O)-F(\emptyset)+L(S)-L(O)\\\leq2F(S) -F(O)+L(S)-L(O), \end{align*}
where the last inequality uses the fact that $F$ is increasing and $F(\emptyset)=0$.
We finish the proof remarking that by definition of Algorithm~\ref{algo:greedy}, ${\pa{L+F}(S_i)-\pa{L+F}(S_{i-1}\cup\sset{\beta(s_i)})}\geq 0$.
\end{proof}
\section{Proof of Theorem~\ref{thm:ratio}}
\label{app:ratio}
\begin{proof}
Let $A$ be the output of Algorithm~\ref{algo:ratio} and let $$\cL_{\kappa_1,\kappa_2}(\lambda,S)\triangleq L_1(S)-\kappa_1 F_1(S) - \lambda\pa{L_2(S)+\kappa_2 F_2(S)}.$$
Recall that $\cL_\kappa=\cL_{\kappa,\kappa}.$
Algorithm~\ref{algo:ratio} satisfies either $\cL_\kappa(0,A_0)\leq 0$  ---  in which case Theorem~\ref{thm:ratio} is trivial since \[\pa{\frac{L_1(A)-(\kappa+\eta) F_1(A)}{L_2(A)+\kappa F_2(A)}}^+=\lambda^\star=0\] --- or $\cL_\kappa(0,A_0)> 0$, in which case we have
\begin{equation} 0> \cL_{\kappa}(\lambda_2,A)\geq\cL_{\kappa+\eta,\kappa}(\lambda_2-\delta,A)\geq\cL_{\kappa+\eta,\kappa}(\lambda_1,A)\geq \cL_{\kappa+\eta,\kappa}(\lambda^\star,A).\label{ratio1}\end{equation}
The first inequality is comes from the update of $\lambda_2$: Notice that before the while loop, we have \[\lambda_2=\frac{L_1(A_0)-F_1(A_0)}{L_2(A_0)+F_2(A_0)}>\frac{L_1(A_0)-\kappa F_1(A_0)}{L_2(A_0)+\kappa F_2(A_0)}>0,\] since $F_2(A_0)>0$, so $0> \cL_{\kappa}(\lambda_2,A_0)$ multiplying by $L_2(A_0)+F_2(A_0)$ on both sides. Notice that in particular, this inequality gives  that $A\neq \emptyset$.

The second inequality follows from
$$\delta=\frac{\eta \min_{\sset{s}\in\actions} F_1(\sset{s})}{L_2(B)+\kappa^2 F_2(B)}\leq \frac{\eta F_1(A)}{L_2(A)+\kappa F_2(A)}\quad \text{since }A\neq \emptyset\text{ and }L_2(B)+\kappa^2 F_2(B)\geq L_2(A)+\kappa F_2(A).$$
Thus, multiplying by $L_2(A)+\kappa F_2(A)>0$, and adding $L_1(A)-\kappa F_1(A)-\lambda_2\pa{L_2(A)+\kappa F_2(A)}$ gives
$$L_1(A)-(\kappa+\eta) F_1(A)-(\lambda_2-\delta)\pa{L_2(A)+\kappa F_2(A)}\leq L_1(A)-\kappa F_1(A)-\lambda_2\pa{L_2(A)+\kappa F_2(A)},$$
i.e.,
$\cL_{\kappa+\eta,\kappa}(\lambda_2-\delta,A)\leq \cL_\kappa(\lambda_2,A).$

The third inequality uses $ \lambda_2-\lambda_1\leq\delta ,$
and the last inequality uses $\lambda_1\leq \lambda^\star$. Indeed, since $\cL_{\kappa}(\lambda_1,S)\geq 0$, the approximation relation given by $\textsc{Algo}_{\kappa},$ 
$$\cL_{\kappa}(\lambda_1,S)\leq  \cL(\lambda_1,O),$$
where $O$ is the minimizer of $\pa{\frac{L_1-F_1}{L_2+F_2}}^{\!\!+}$ (for the constraints considered by $\textsc{Algo}_{\kappa}$), gives $0\leq \cL(\lambda_1,O)$.
Thus, $$\cL^+(\lambda_1,O)\triangleq\pa{L_1(O)-F_1(O)}^+-\lambda_1\pa{L_2(O)+F_2(O)}\geq \cL(\lambda_1,O)\geq 0.$$
Finally, since $L_2(O)+F_2(O)> 0$ ($O\neq \emptyset$), we have $\lambda_1\leq \lambda^\star.$

In \eqref{ratio1}, since $A\neq \emptyset$, we have $\frac{L_1(A)-(\kappa+\eta) F_1(A)}{L_2(A)+\kappa F_2(A)}\leq\lambda^\star$ and therefore, $\pa{\frac{L_1(A)-(\kappa+\eta) F_1(A)}{L_2(A)+\kappa F_2(A)}}^+\leq\lambda^\star$.

The time complexity for the binary search is $\cO(\log(1/\delta))\leq\cO(\log( m t/\eta))$ for $\cC_t$ given by any algorithm of Table~\ref{table:algos}, and $F(A)=\max_{\bdelta\in \cC_t^+-\vmean{t-1}}\be_A\transpose\bdelta$.
\end{proof}

\end{document}